\documentclass{article}

  \usepackage[final,nonatbib]{neurips_2020}

\usepackage[utf8]{inputenc} 
\usepackage[T1]{fontenc}    
\usepackage{hyperref}       
\usepackage{url}            
\usepackage{booktabs}       
\usepackage{amsfonts}       
\usepackage{nicefrac}       
\usepackage{microtype}      
\usepackage{amssymb,amsfonts,amsmath,amsthm}
\usepackage{graphicx}
\usepackage{subfig}
\usepackage{mathtools}
\usepackage{enumitem}
\usepackage{color}

\newtheorem{thm}{Theorem}
\newtheorem{clm}{Claim}
\newtheorem{assume}{Assumption}
\def\R{\mathbb{ R}}
\def\N{\mathbb{ N}}
\def\bone{\mathbf{1}}
\def\bzero{\mathbf{0}}
\def\bx{\mathbf{x}}
\def\hatc{\hat{c}}
\def\hath{\hat{h}}

\def\E{\mathbb{ E}}
\def\calX{\mathcal{ X}}
\def\hatcalR{\hat{\mathcal{R}}}
\def\hatcalU{\hat{\mathcal{U}}}
\def\hatcalS{\hat{\mathcal{S}}}
\def\calE{\mathcal{ E}}
\def\calS{\mathcal{ S}}
\def\calA{\mathcal{ A}}
\def\calT{\mathcal{ T}}
\def\calD{\mathcal{ D}}

\def\calC{\mathcal{ C}}
\def\calH{\mathcal{ H}}

\def\calM{\mathcal{ M}}
\def\calL{\mathcal{ L}}
\def\calU{\mathcal{ U}}
\def\calY{\mathcal{ Y}}
\def\calB{\mathcal{ B}}
\def\hatmu{\hat{\mu}}

\def\hats{\hat{s}}
\def\hatt{\hat{t}}
\def\baru{\bar{u}}
\def\calP{\mathcal{P}}
\def\hatSigma{\hat{\Sigma}}
\newcommand{\thead}[1]{\multicolumn{1}{c}{\textbf{#1}}}
\newcommand{\lhead}[1]{\multicolumn{1}{l}{\textbf{#1}}}
\newcommand{\argmin}{\mathop{\rm argmin}\limits}
\newcommand\norm[1]{\left\lVert#1\right\rVert}
\newcommand{\defmin}{\mathop{\rm min}\limits}

\newcommand{\inner}[1]{\left\langle#1\right\rangle}

\title{Statistical Optimal Transport posed as\\ Learning Kernel Mean Embedding}

\author{%
  J.~Saketha~Nath\\
  Department of Computer Science and Engineering,\\
  Indian Institute of Technology Hyderabad, INDIA.\\
  \texttt{saketha@cse.iith.ac.in} \\
   \And
   Pratik Jawanpuria\\
  Microsoft IDC,\\
  Hyderabad, INDIA. \\
  \texttt{pratik.jawanpuria@microsoft.com}\\
}

\begin{document}

\maketitle

\begin{abstract}
The objective in statistical Optimal Transport (OT) is to consistently estimate the optimal transport plan/map solely using samples from the given source and target marginal distributions. This work takes the novel approach of posing statistical OT as that of learning the transport plan's kernel mean embedding from sample based estimates of marginal embeddings. The proposed estimator controls overfitting by employing maximum mean discrepancy based regularization, which is complementary to $\phi$-divergence (entropy) based regularization popularly employed in existing estimators. A key result is that, under very mild conditions, $\epsilon$-optimal recovery of the transport plan as well as the Barycentric-projection based transport map is possible with a sample complexity that is completely dimension-free. Moreover, the implicit smoothing in the kernel mean embeddings enables out-of-sample estimation. An appropriate representer theorem is proved leading to a kernelized convex formulation for the estimator, which can then be potentially used to perform OT even in non-standard domains. Empirical results illustrate the efficacy of the proposed approach.

\end{abstract}
\section{Introduction}
Optimal Transport is proving to be an increasingly successful tool in solving diverse machine learning problems. Recent research shows that variants of Optimal Transport (OT) achieve state-of-the-art performance in various machine learning (ML) applications such as data alignment/integration~\cite{alvarez18wordEmb,NIPS2019_9501,tam-etal-2019-optimal,pmlr-v108-janati20a}, domain adaptation~\cite{Courty17domAda,pmlr-v89-redko19a}, model interpolation/combination~\cite{solomon15graphics,bioapp19,dognin2018wasserstein}, natural language processing~\cite{Xu2018DistilledWL,yurochkin2019hierarchical} etc. It is also shown that OT based (Wasserstein) metrics serve as good loss functions in both supervised~\cite{Frogner15,mjaw20a} and  unsupervised~\cite{Genevay2017LearningGM} learning.

Given two marginal distributions over source and target domains, and a cost function between elements of the domains, the classical OT problem (Kantorovich's formulation) is that of finding the joint distribution whose marginals are equal to the given marginals, and which minimizes the expected cost with respect to this joint distribution~\cite{KatoroOT}. This joint distribution is known as the (optimal) transport plan or the optimal coupling. A related object of interest for ML applications is the so-called Barycentric-projection based transport map corresponding to a transport plan (e.g., refer (11) in \cite{seguy2018large}). Though OT techniques already improve state-of-the-art in many ML applications, there are two main bottlenecks that seem to limit OT's success in ML settings:
\begin{itemize}[noitemsep,topsep=0pt,leftmargin=*]
    \item while continuous distributions are ubiquitous, algorithms for finding the transport plan/map over continuous domains are very scarce~\cite{Genevay16sgd}. The situation is worse in case of non-standard domains, which are not uncommon in ML.
    \item the marginal distributions are never available, and merely samples from them are given. The variant of OT where the transport plan/map needs to be estimated merely using samples from the marginals is known as the statistical OT problem. Unfortunately, this estimation problem is plagued with the curse of dimensionality: the sample complexity of $O(m^{-1/d})$, where $m$ is number of samples, and $d$ is the dimensionality of data, cannot be improved without further assumptions~\cite{nilesweed2019estimation}.
\end{itemize}

 Though several works alleviated the curse of dimensionality~\cite{Genevay19sampComp,nilesweed2019estimation,hutter2020minimax}, none of them completely remove the adversarial dependence on dimensionality. Further, authors in~\cite{Genevay16sgd,Genevay19sampComp,Forrow19facCoup,repwork} comment that estimators that are free from the curse of dimensionality are important, yet not well-studied. The concluding report from a recent workshop on OT~(refer section 2~in~\cite{repwork}) summarizes that one of the major open problems in this area is to design estimators in context of continuous statistical OT whose sample complexity is not a strong function of the dimension (ideally dimension-free).

Our work focuses on this challenging and important problem of statistical OT over continuous domains, and seeks consistent estimators for $\epsilon$-optimal transport plan/map, whose sample complexity is dimension-free. To this end, we recall the observation that the estimation in statistical continuous OT is as hard as that of estimating the involved densities~\cite{nilesweed2019estimation}. Since estimating densities is challenging in high dimensions~\cite{dudley1969}, estimating OT plan/map also turns out to be equally challenging. Here we make the critical observation that though densities are known to be hard to estimate, the corresponding kernel mean embeddings~\cite{MAL-060} are easy to estimate. Infact, their sample complexity, $O(1/\sqrt{m})$, is completely independent of the dimensionality.

Motivated by this, we take the novel approach of equivalently re-formulating the statistical OT problem solely in terms of the relevant kernel mean embeddings. More specifically, our formulation finds the (characterizing) kernel mean embedding of a joint distribution with least expected cost, and whose marginal embeddings are close to the given-sample based estimates of the marginal embeddings. There are several advantages of this new approach:
\begin{enumerate}[noitemsep,topsep=0pt,leftmargin=*]
    \item because the samples based estimates of the kernel mean embeddings of the marginals are known to have sample complexities that are dimension-free, it is expected that the sample complexity remains dimension-free even for the proposed estimator of the transport plan embedding. This is illustrated in figure~\ref{fig:motive}.
    \item kernel embeddings provide implicit smoothness, as controlled by the kernel. Appropriate smoothness not only improves the quality of estimation, but also enable out-of-sample estimation.
    \item since Maximum Mean Discrepency (MMD) is the natural notion of distance in the kernel mean embedding space, this reformulation facilitates MMD based regularization for controlling overfitting. Such regularizers are complementary to the $\phi$-divergence (or entropy) based regularizers popularly employed in existing estimators. ~\cite{Sriperumbudur09onintegral} oberves that MMD and $\phi$-divergence based regularization exhibit complementary properties and hence both are interesting to study.
\end{enumerate}
\begin{figure}
        \center{\includegraphics[width=\textwidth]
        {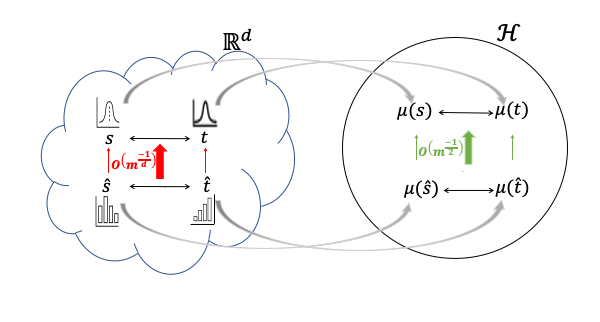}}
        \caption{\label{fig:motive} Illustration of the key idea in the proposed methodology. $s,t$ represent the densities and $\hats,\hatt$ denote their sample-based estimates. $\mu(s),\mu(t)$ denote their kernel mean embeddings, while the sample-based estimates of the embeddings are denoted by $\mu(\hats),\mu(\hatt)$. Since sample-based embeddings are known to converge faster, OT is expected to converge faster in the embedding space.}
      \end{figure}
      
A key result from this work is that, under very mild conditions, the proposed methodology can recover an $\epsilon$-optimal transport plan and corresponding (Barycentric-projection based) transport map with a sample complexity, $O(m^{-1/2})$, which is completely dimension-free~\footnote{Here we use a simplified notation for $O(m^{-1/2})$, where the terms involving $\epsilon$ are ignored. Nevertheless, all the terms/constants ignored are indeed independent of the dimension (and $m$).}. Another contribution is an appropriate representer theorem that guarantees finite characterization for the transport plan embedding, leading to a fully kernelized and convex formulation for the estimation. Thus the same formulation can potentially be used for obtaining estimators with all variants of OT: continuous, semi-discrete, and discrete, merely by switching the kernel between the Kronecker delta  and the Gaussian kernels. More importantly, the same can be used to solve OT problems in non-standard domains using appropriate universal kernels~\cite{univker10}. Finally, we discuss special cases where the proposed convex formulation can be solved efficiently using ADMM, Conditional gradient etc. Empirical results on synthetic and real-world datasets illustrate the efficacy of the proposed approach.


\section{Background on OT and Kernel Embeddings}\label{sec:backg}
In this section we briefly summarize the theories of optimal transport and kernel mean embeddings, which are essential for understanding the proposed methodology. Familiar readers may skip this section entirely.

\subsubsection*{Optimal Transport}
We begin with a brief on optimal transport. For more details, please refer~\cite{CompOT}, which is a comprehensive monologue on the subject with focus on recent developments related to machine learning.

Let $\calX,\calY$ be any two sets that form locally compact Hausdorff topological spaces. We denote the set of all Radon probability measures over $\calX$ by $\calM^1(\calX)$; whereas we denote the set of strictly positive measures by $\calM^1_+(\calX)$ . Let $c:\calX\times\calY$ denote a function that evaluates the cost between elements in $\calX,\calY$ and let $p_s\in\calM^1_+(\calX),p_t\in\calM^1_+(\calY)$.  Then, the Kantorovich's OT formulation~\cite{KatoroOT} is:
\begin{equation}\label{eqn:Kantorov}
\begin{array}{ll}
\min_{\pi\in\calM^1\left(\calX,\calY\right)}&\int c(x,y)\ \textup{d}\pi(x,y),\\
\textup{s.t.}& \pi^\calX = p_s, \pi^\calY=p_t,
\end{array}
\end{equation}
where $\pi^\calX,\pi^\calY$ denote the marginal measures of $\pi$ over $\calX,\calY$ respectively. An optimal solution of (\ref{eqn:Kantorov}) is referred to as an optimal transport plan or optimal coupling. 

\textbf{Statistical OT}: In the setting of statistical OT, the marginals $p_s,p_t$ are not available; however, iid samples from them are given. 
Let $\calD_x=\left\{x_1,\ldots,x_m\right\}$ denote the set of $m$ iid samples from $p_s$ and let  $\calD_y=\left\{y_1,\ldots,y_n\right\}$ denote $n$ iid samples from $p_t$. The cost function is known only at the sample data points. Let $\calC\in\R^{m\times n}$ denote the cost matrix with with $(i,j)^{th}$ entry as $c(x_i,y_j)$. 

A popular way to estimate the optimal plan in (\ref{eqn:Kantorov}) is to simply employ the sample based plug-in estimates for the marginals: $\hat{p}_s\equiv\frac{1}{m}\sum_{i=1}^m\delta_{x_i}$ and $\hat{p}_t\equiv\frac{1}{n}\sum_{j=1}^n\delta_{y_j}$, in place of the true (unknown) marginals. Here, $\delta$ denotes the Dirac delta function. In such a case, (\ref{eqn:Kantorov}) simplifies as the standard discrete OT problem:
\begin{equation}\label{eqn:discOT}
\begin{array}{ll}
\min_{\pi\in\R^{m\times n}}&tr(\pi\calC^\top),\\
\textup{s.t.}& \pi\bone = \frac{1}{m}\bone, \pi^\top\bone=\frac{1}{n}\bone, \pi\ge\bzero,
\end{array}
\end{equation}
where $tr(M)$ is the trace of matrix $M$, and $\bone,\bzero$ denote vectors/matrices with all entries as unity, zero respectively (of appropriate dimension). Since the sample complexity of (\ref{eqn:discOT}) in estimating (\ref{eqn:Kantorov}) is prohibitively high for high-dimensional domains~\cite{nilesweed2019estimation}, alternative estimation methods are sought after.

\subsubsection*{Kernel mean embeddings}
This section presents a brief on the theory of kernel mean embeddings. For more details, please refer~\cite{Song09condembd}. Let $k$ be a kernel defined over a domain $\calX$ and let $\phi_k,\calH_k$ be the kernel's canonical feature map and the canonical RKHS. Then, the kernel mean embedding of a random variable $X$ is defined as $\mu_X\equiv\E\left[\phi(X)\right]$. The embedding $\mu_X$ is well-defined, and $\mu_X\in\calH_k$, whenever $k$ is normalized. Further, if (and only if) $k$ is a characteristic kernel~\cite{SriperumbudurFL11}, then the map $X\mapsto\mu_X$ is one-to-one. For discrete probability measures, the Kronecker delta kernel is characteristic, while for continuous measures over $\R^d$, the Gaussian kernel is an example of a characteristic kernel. Using these embeddings, one can compute expectations of functions of the respective random variables, whenever they exist: for e.g., $\E[f(X)]=\E[\langle f,\phi(x)\rangle]_{\calH_k}=\langle f,\E[\phi(X)]\rangle_{\calH_k}=\langle f,\mu_X\rangle_{\calH_k}\forall\ f\in\calH_k$.

The notion of kernel mean embeddings easily extends to the case of jointly defined random variables. Let $k_1,k_2$ be two kernels defined over domains $\calX,\calY$ respectively. Let $\phi_1,\phi_2$ be the corresponding canonical feature maps and let $\calH_1,\calH_2$ be the canonical RKHSs. Then the cross-covariance operator (the joint embedding) is defined as $C_{XY}\equiv\E\left[\phi_1(X)\otimes\phi_2(Y)\right]$, where $\otimes$ denotes the tensor product. Again, whenever $k_1,k_2$ are individually characteristic, the map $(X,Y)\mapsto C_{XY}$ is one-to-one and $\E[h(X,Y)]=\langle h,C_{XY}\rangle_{\calH_{1}\otimes\calH_{2}}\ \forall\ h\in\calH_1\otimes\calH_2$. Analogously, one can define the auto-covariance operator $C_{XX}\equiv\E\left[\phi_1(X)\otimes\phi_1(X)\right]$.

The notion of embedding conditionals is also straight-forward: $\mu_{Y/x}\equiv\E\left[\phi_2(Y)/x\right]$. Additionally, one defines a conditional embedding operator $C_{Y/X}:\calH_1\mapsto\calH_2$, such that $C_{Y/X}\left(\phi_1(x)\right)=\mu_{Y/x}\ \forall\ x\in\calX$. For convenience of notation, $C_{Y/X}\left(\phi_1(x)\right)$ is simplified as $C_{Y/X}\phi_1(x)$. With this definition, one can show that the relation $C_{Y/X}C_{XX}=C_{YX}$ holds. Also, the kernel sum rule~\cite{Song09condembd} relates the conditional operator to the mean embeddings: $\mu_Y=C_{Y/X}\mu_X$.

We end with this brief with a note on the related notion of universal kernel~\cite{SriperumbudurFL11}. A kernel defined over a domain $\calX$ is universal iff its RKHS is dense in the set of all continuous functions over $\calX$~\cite{univker10}. If $k_1,k_2$ are universal over $\calX,\calY$ respectively, then $k=k_1k_2$ is universal over $\calX\times\calY$. Moreover, $\phi_k(x,y)=\phi_{1}(x)\otimes\phi_{2}(y)\ \forall\ x\in\calX,\ y\in\calY$. Hence, the RKHS $\calH_1\otimes\calH_2$ is dense in the set of all continuous functions over $\calX\times\calY$. Finally, universal kernels are also characteristic.

\section{Proposed Methodology}\label{sec:main}



We begin by re-formulating (\ref{eqn:Kantorov}) solely in terms of kernel mean embeddings and operators. Let $k_1,k_2$ be characteristic  kernels defined over $\calX,\calY$ respectively. Let $\phi_1,\phi_2$ and $\calH_1,\calH_2$ denote the canonical feature maps and the reproducing kernel Hilbert spaces (RKHS) corresponding to the kernels $k_1,k_2$ respectively. Let $\langle\cdot,\cdot\rangle_{\calH},\|\cdot\|_{\calH}$ denote the default inner-product, norm in the RKHS $\calH$. Let $\mu_{s}\equiv\E_{X\sim p_s}\left[\phi_1(X)\right],\ \mu_{t}\equiv\E_{Y\sim p_t}\left[\phi_2(Y)\right]$ denote the kernel mean embeddings of the marginals $p_s,p_t$ respectively. Let $\Sigma_{ss}\equiv\E_{X\sim p_s}\left[\phi_1(X)\otimes\phi_1(X)\right]$ and $\Sigma_{tt}\equiv\E_{Y\sim p_t}\left[\phi_2(Y)\otimes\phi_2(Y)\right]$ denote the auto-covariance embeddings of $p_s,p_t$ respectively. Recall that $\otimes$ denotes tensor product. 

Since the variable in (\ref{eqn:Kantorov}), $\pi$, is a joint measure, the cross-covariance operator, $\calU=\E_{(X,Y)\sim\pi}\left[\phi_1(X)\otimes\phi_2(Y)\right]$, is the suitable kernel mean embedding to be employed. However, since the constraints involve the marginals of $\pi$, denoted by $\pi_1,\pi_2$; it is natural to employ the kernel sum rule~\cite{Song09condembd}, which relates the cross-covariance operator to the marginal embeddings, $\mu_1,\mu_2$, via the conditional embedding operators, $\calU_1,\calU_2$, and the auto-covariance operators, $\Sigma^\calU_1,\Sigma^\calU_2$. 
The relations between these operators and embeddings follow from the definition of conditional embedding and the kernel sum rule \cite{Song09condembd}:
\begin{equation}\label{eqn:ksr}
\calU=\Sigma_{1}^\calU\left(\calU_1\right)^\top=\calU_2\Sigma_{2}^\calU,\ \calU_1\mu_{1}=\mu_{2},\ \calU_2\mu_{2}=\mu_{1}.
\end{equation}
Here, $M^\top$ denotes the adjoint of $M$.


In order to re-write the objective using the above operators, we assume that the cost function, $c(\cdot,\cdot)$, can be embedded in $\calH_1\otimes\calH_2$. This assumption is trivially true if the domains are discrete. However, in case of continuous domains this need not be true, in general. Hence we additionally assume that the kernel(s) corresponding to continuous domain(s) is(are) universal and that the cost function, $C(\cdot,\cdot)$, is continuous in that(those) continuous variable(s). It then follows that $c(\cdot,\cdot)$ can be arbitrarily closely approximated by elements in $\calH_1\otimes\calH_2$~\cite{SriperumbudurFL11}. Note that universal kernels are well-studied and known for non-standard domains too~\cite{univker10}. These very mild assumptions are summarized below:
\begin{assume}\label{ass:one}
Both kernels $k_1,k_2$ are characteristic. Moreover, if $k_i$ is over a continuous domain, then it is universal.
\end{assume}
\begin{assume}\label{ass:two}
We assume that $c\in\calH_1\otimes\calH_2$, where $c$ denotes either the exact function or the (arbitrarily) close approximation of it that can be embedded.
\end{assume}

Note that the objective in (\ref{eqn:Kantorov}) can be written as: $\E\left[c(X,Y)\right]=\langle c,\calU\rangle_{\calH_1\otimes\calH_2}$. Using this and (\ref{eqn:ksr}), leads to the following kernel embedding formulation for OT:
\begin{equation}\label{eqn:exactemb}
\begin{array}{cl}
\defmin_{\calU\in\calE_{21},\calU_1\in\calL_{12},\calU_2\in\calL_{21}}&\langle c,\calU\rangle_{\calH_1\otimes\calH_2}\\
\textup{s.t.\ \ \ \ }
&\calU_1\mu_s=\mu_t,\ \calU_2\mu_t=\mu_s,\\
&\calU=\Sigma_{ss}\calU_1^\top, \Sigma^\calU_1=\Sigma_{ss},\ \calU=\calU_2\Sigma_{tt}, \Sigma^\calU_2=\Sigma_{tt},
\end{array}
\end{equation}
where $\calL_{ij}$ is the set of all linear operators from $\calH_i\mapsto\calH_j$, and $\calE_{21}\equiv\left\{\calU\in\calL_{21}\ |\ \exists p\in\calM^1(\calX,\calY)\ \ni\ \calU=\E_{(X,Y)\sim p}\left[\phi_1(X)\otimes\phi_2(Y)\right]\right\}$ is the set of all valid cross-covariance operators. Note that the constraints $\calU=\Sigma_{ss}\calU_1^\top,\Sigma^\calU_1=\Sigma_{ss}\Rightarrow\calU=\Sigma^\calU_1\calU_1^\top$, which in turn gives that $\calU_1$ is a valid conditional embedding associated with $\calU$. However, we keep the former couple of constraints rather than the later one because i) there is no loss of generality ii) they will lead to an elagant representor theorem, kernelization and efficient optimization, as will be clear later. Analogous comments hold for the couple $\calU=\calU_2\Sigma_{tt},\Sigma^\calU_2=\Sigma_{tt}$.

The equivalence of (\ref{eqn:exactemb}) and (\ref{eqn:Kantorov}) follows from the one-to-one correspondence between the measures involved and their kernel embeddings, which is guaranteed by the characteristic kernels, and from the crucial embedding characterizing constraint: $\calU\in\calE_{21}$. Without this  characterizing constraint, the formulation is not meaningful.  We summarize the above re-formulation in the following theorem:

\begin{thm}
Under Assumptions~\ref{ass:one}-\ref{ass:two}, the Kantorovich formulation of OT, (\ref{eqn:Kantorov}), is  equivalent to (\ref{eqn:exactemb}).
\end{thm}

Note that unlike existing formulae for the operator embeddings~\cite{Song09condembd}, which eliminate two of the three operators $\calU,\calU_1,\calU_2$; we critically preserve all of them in (\ref{eqn:exactemb}). This is because they facilitate efficient regularization in the statistical estimation set-up and lead to efficient algorithms (as will be shown later). Also, the characterization of embedding, $\calE_{21}$, is included only for the cross-covariance, and not explicitly included for the conditional operators. This is fine because the conditional operators are well-defined given the cross-covariance, and the auto-covariances.

The key advantage of the proposed formulation (\ref{eqn:exactemb}) over (\ref{eqn:Kantorov}) is that the sample based estimates for kernel mean embeddings of the marginals, which are known to have dimension-free sample complexities, can be employed directly in the statistical OT setting.

\subsection{Re-formulation as Learning Embedding problem}
As motivated earlier, we aim to employ the standard sample based estimates for the kernel mean embeddings of the marginals in the re-formulation (\ref{eqn:exactemb}). To this end, let the estimates for the marginal kernel mean embeddings be denoted by: $\hatmu_s\equiv\frac{1}{m}\sum_{i=1}^m\phi_1(x_i)$ and $\hatmu_t\equiv\frac{1}{n}\sum_{j=1}^n\phi_2(y_j)$. Likewise, the estimates of the auto-covariance embeddings are given by $\hatSigma_{ss}\equiv\frac{1}{m}\sum_{i=1}^m\phi_1\left(x_i\right)\otimes\phi_1\left(x_i\right)$ and $\hatSigma_{tt}\equiv\frac{1}{n}\sum_{j=1}^n\phi_2\left(y_j\right)\otimes\phi_2\left(y_j\right)$.

In the statistical OT setting, the cost function, $c$, is only available at the given samples. In continuous domains, there will exist many functions in the RKHS that will exactly match $c$, when restricted to the samples. Each such choice will lead to a valid estimator. We choose $\hatc$ to be the orthogonal projection of $c$ onto the samples: $\hatc\equiv\sum_{i=1}^m\sum_{j=1}^n\rho^*_{ij}\phi_1(x_i)\otimes\phi_2(y_j)$, where $\rho^*\equiv\arg\min_{\rho}\left\|c-\sum_{i=1}^m\sum_{j=1}^n\rho_{ij}\phi_1(x_i)\otimes\phi_2(y_j)\right\|_{\calH_1\otimes\calH_2}$ and $\|\cdot\|_{\calH_1\otimes\calH_2}$ is the Hilbert-Schmidt operator norm. Straight-forward computation shows that $\rho^*=\left(G_1\odot G_2\right)^{-1}\calC$, where $G_1$ and $G_2$ are the gram-matrices with $k_1$ and $k_2$ over the samples $x_1,\ldots,x_m$ and $y_1,\ldots,y_n$ respectively, and $\odot$ denotes the element-wise product. Recall that $\calC\in\R^{m\times n}$ denotes the cost matrix with with $(i,j)^{th}$ entry as $c(x_i,y_j)$. For universal kernels, it follows that $\hatc$ will be equal to $c$ at the given samples, and hence the above is a valid choice for estimation. In addition, the above choice of $\hatc$ helps us in proving the representer theorem~(Theorem~\ref{thm:repthm}). 


Now, employing these estimates in (\ref{eqn:exactemb}) must be performed with caution as i) the equality constraints now will be in the (potentially infinite dimensional) RKHS, ii) more importantly, matching the estimates exactly will lead to overfitting. Hence, we propose to introduce appropriate regularization by insisting that there is a close match rather than an exact match. This leads to the following kernel mean embedding learning formulation:
\begin{equation}\label{eqn:empemb0}
\begin{array}{cl}
\defmin_{\calU\in\calE_{21},\calU_1\in\calL_{12},\calU_2\in\calL_{21}}&\langle\hatc,\calU\rangle_{\calH_1\otimes\calH_2}\\
\textup{s.t.\ \ \ \ \ \ }&\left\|\calU_1\hatmu_s-\hatmu_t\right\|_{\calH_2}\le\Delta_1,\ \left\|\calU_2\hatmu_t-\hatmu_s\right\|_{\calH_1}\le\Delta_2,\\
&\left\|\calU-\hatSigma_{ss}\calU_1^\top\right\|_{\calH_1\otimes\calH_2}\le\vartheta_1,\left\|\calU-\calU_2\hatSigma_{tt}\right\|_{\calH_1\otimes\calH_2}\le\vartheta_2,\\
&\|\Sigma_1^\calU-\hatSigma_{ss}\|_{\calH_1\otimes\calH_1}\le\zeta_1, \|\Sigma_2^\calU-\hatSigma_{tt}\|_{\calH_2\otimes\calH_2}\le\zeta_2,
\end{array}
\end{equation}
where $\Delta_1,\Delta_2,\vartheta_1,\vartheta_2,\zeta_1,\zeta_2$ are regularization hyper-parameters introduced to prevent overfitting to the estimates.
\subsection{Statistical Analysis of the Learning Formulation}
The proposed embedding learning formulation (\ref{eqn:empemb0}) is an approximation to the OT problem~(\ref{eqn:exactemb}) because of two reasons: i) the regularization hyper-parameters $\Delta_1,\Delta_2,\vartheta_1,\vartheta_2,\zeta_1,\zeta_2$, which are non-zero (positive) ii) sample-based estimates $(\hatc,\hatmu_s,\hatmu_t,\hatSigma_{ss},\hatSigma_{tt})$ are employed. While the effect of the former is clear, for e.g., as the hyper-parameters $\rightarrow0$, the approximation error, $\epsilon$, goes to zero; the sample complexity for estimation is more insightful. To this end we present the following theorem:
\begin{assume}\label{ass:three}
Let us assume that the kernels are normalized/bounded i.e., $\max_{x\in\calX}k_1(x,x)=1, \max_{y\in\calY}k_2(y,y)=1$.
\end{assume}
\begin{thm}\label{thm:sampcomp}
Let $g\left(\hatc,\hatmu_s,\hatmu_t,\hatSigma_{ss},\hatSigma_{tt}\right)$ denote the optimal objective of (\ref{eqn:empemb0}) in Tikhonov form. Under Assumptions \ref{ass:one}-\ref{ass:three}, with high probability we have that, $\left|g\left(\hatc,\hatmu_s,\hatmu_t,\hatSigma_{ss},\hatSigma_{tt}\right)-g\left(c,\mu_s,\mu_t,\Sigma_{ss},\Sigma_{tt}\right)\right|\le O\left(1/{\sqrt{\min(m,n)}}\right)$. The constants in the RHS of the inequality are dimension-free.
\end{thm}
Theorem~\ref{thm:sampcomp} shows that embedding of an $\epsilon$-optimal transport plan can recovered by solving (\ref{eqn:empemb0}) with a sample complexity that is dimension-free. The proof of this theorem is detailed in~Appendix~\ref{app:samplecomplexity}. The idea is to uniformly bound the difference between the population and sample versions of each of the terms in the objective. Interestingly, each of these difference terms can either be bounded by relevant estimation errors in embedding space or by approximation errors in the RKHS, both of which are known to be dimension-free.

Note that the regularization in (\ref{eqn:empemb0}) is based on the Maximum Mean Discrepancy (MMD) distances between the kernel embeddings. This characteristic of our estimators is in contrast with the popular entropic regularization~\cite{sinkhorn13}, or $\phi$-divergence based regularization~\cite{Liero2018} in existing OT estimators.~\cite{Sriperumbudur09onintegral} argue that MMD and $\phi$-divergence based regularization have complementary properties. Hence both are interesting to study. While the dependence on dimensionality is adversely exponential with entropic regularization, if accurate solutions are desired~\cite{Genevay19sampComp}, 
the proposed MMD based regularization for statistical OT leads to dimension-free estimation.

\subsection{Representer theorem \& Kernelization}
Interestingly, (\ref{eqn:empemb0}) admits a finite parameterization facilitating it's efficient optimization. This important result is summarized in the representer theorem below:
\begin{thm}\label{thm:repthm}
Whenever (\ref{eqn:empemb0}) is solvable, there exists an optimal solution, $\calU^*,\calU^*_1,\calU^*_2$, of (\ref{eqn:empemb0}) such that $\calU^*=\sum_{i=1}^m\sum_{j=1}^n\alpha_{ij}\phi_1(x_i)\otimes\phi_2(y_j), \calU^*_1=\sum_{i=1}^m\sum_{j=1}^n\beta_{ji}\phi_2(y_j)\otimes\phi_1(x_i),\calU^*_2=\sum_{i=1}^m\sum_{j=1}^n\gamma_{ij}\phi_1(x_i)\otimes\phi_2(y_j)$. Here $\alpha\in\R^{m\times n},\beta\in\R^{n\times m},\gamma\in\R^{m\times n}$ that are an optimal solution for the kernelized and convex formulation (\ref{eqn:final}) given below:
\begin{equation}\label{eqn:final}
{\def\arraystretch{1.3}
\begin{array}{cl}
    \defmin_{\alpha,\gamma\in\R^{m\times n},\beta\in\R^{n\times m}} & tr(\alpha\calC^\top)\\
    \textup{s.t.}& \frac{1}{m^2}\bone^\top G_1\beta^\top G_2\beta G_1\bone-\frac{2}{mn}\bone^\top G_2\beta G_1\bone+\frac{1}{n^2}\bone^\top G_2\bone\le\Delta_1^2\\
    & \frac{1}{n^2}\bone^\top G_2\gamma^\top G_1\gamma G_2\bone-\frac{2}{mn}\bone^\top G_1\gamma G_2\bone+\frac{1}{m^2}\bone^\top G_1\bone\le\Delta_2^2\\
    & \left\langle G_1\alpha-\frac{1}{m}G_1^2\beta^\top,\alpha G_2-\frac{1}{m}G_1\beta^\top G_2\right\rangle_F\le\vartheta_1^2,\\
    & \inner{\alpha G_2 -\frac{1}{n}\gamma G_2^2,G_1\alpha-\frac{1}{n}G_1\gamma G_2}_F\le\vartheta_2^2,\\
    &\|\alpha\bone-\frac{1}{m}\bone\|^2_{G_1\odot G_1}\le\zeta_1^2,\|\alpha^\top\bone-\frac{1}{n}\bone\|^2_{G_2\odot G_2}\le\zeta_2^2\\
    &\alpha\ge0,\bone^\top\alpha\bone=1,
\end{array}
}
\end{equation}
where, $G_1$ and $G_2$ are the gram-matrices with $k_1$ and $k_2$ over $x_1,\ldots,x_m$ and $y_1,\ldots,y_n$ respectively, and $\|\bx\|_M^2\equiv \bx^\top M\bx$, is the mahalanobis squared-norm of $\bx$.
\end{thm}
The proof of this theorem is detailed in Appendix~\ref{app:repthm}. Apart from standard representer theorem-type arguments, the proof includes arguments that show that the characterizing set $\calE_{21}$ when restricted to the linear combinations of embeddings is exactly same as the convex combinations of those. This helps us replace the membership to $\calE_{21}$ constraint by a simplex constraint.

We note that (\ref{eqn:final}) is jointly convex in the variables $\alpha,\beta,$ and $\gamma$. This is because the constraints are either convex quadratic or linear and the objective is also linear. Hence obtaining estimators using (\ref{eqn:final}) is computationally tractable (refer also section~\ref{sec:special}). It is easy to verify that (\ref{eqn:final}) simplifies to the discrete OT problem (\ref{eqn:discOT}) if both the kernels are chosen to be the Kronecker delta and all the hyper-parameters are set to zero. If one of the kernel is chosen as the Kronecker delta and the other as the Gaussian kernel, then (\ref{eqn:final}) can be used for semi-discrete OT in the statistical setting. Additionally, by employing appropriate universal kernels, (\ref{eqn:final}) can  be used for statistical OT in non-standard domains.

We end this section with a small technical note. While the cross-covariance operator obtained by solving (\ref{eqn:final}) will always be a valid one; for some hyper-parameters, which are too high, it may happen that the optimal $\beta,\gamma$ induce invalid conditional embeddings. This may make computing the transport map~(\ref{eqn:inf}) intractable. Hence, in practice, we include additional constraints $\beta,\gamma\ge0$.

\subsection{Proposed Optimal Map Estimator}\label{sec:optmap}
Once the embedding of the transport plan is obtained by solving (\ref{eqn:final}), generic approaches for recovering the measure corresponding to a kernel embedding, detailed in~\cite{kanagawa14,conden19}, can be employed to recover the corresponding transport plan. Moreover, since the recovery methods in~\cite{conden19} have dimension-free sample complexity, the overall sample complexity for estimating the optimal transport plan hence remains dimension-free.

We estimate the Barycentric-projection based optimal transport map, $\calT$, at any $x\in\calX$ as follows:
\begin{equation}
\begin{array}{lll}
\calT(x)\equiv\argmin_{y\in\calY}\E\left[c\left(y,Y\right)/x\right]&=& \argmin_{y\in\calY}\left\langle c(y,\cdot),\calU_1^*\phi_1(x)\right\rangle,\\
&=&\argmin_{y\in\calY}\sum_{j=1}^n\left(c(y,y_j)\sum_{j=1}^n\left(\beta^*_{ji}k_1\left(x_i,x\right)\right)\right),\label{eqn:inf}
\end{array}
\end{equation}
where $\beta^*$ are obtained by solving (\ref{eqn:final}) and $\calU_1^*$ is the corresponding conditional embedding.
(\ref{eqn:inf}) turns out to be that of finding the Karcher mean~\cite{karcher2014riemannian}, whenever the cost is a squared-metric etc. 
Alternatively, one can directly minimize $\E\left[c\left(y,Y\right)/x\right]$ with respect to $y\in\calY$ using stochastic gradient descent (SGD). 
The following theorem summarizes the consistency with SGD:
\begin{thm}\label{thm:sgd}
Let the cost be a metric or it's powers greater than unity and let $\calY$ be compact. Then the SGD based estimator for $\calT$ has a sample complexity that remains dimension-free.
\end{thm}
The proof of this theorem is detailed in Appendix~\ref{app:sgd} and follows from standard results in stochastic convex optimization.

An advantage with our map estimator is that it can be computed even at out-of-sample $x\in\calX$. This is possible because of the implicit smoothing induced by the kernel.




\subsection{Some Special Cases}\label{sec:special}

Though (\ref{eqn:final}) can be solved using off-the-shelf convex solvers, the structure in the proposed problem can be exploited to derive further efficient solvers. Further speed-up may be obtained in the special case when $\vartheta_i=0$ in (\ref{eqn:final}). This simplifies the constraints corresponding to $\vartheta_1$ and $\vartheta_2$ in (\ref{eqn:final}) as $\alpha=(1/m)G_1\beta^\top$ and  $\alpha=(1/n)\gamma G_2$, respectively. Using this and re-writing (\ref{eqn:final}) in Tikhonov form leads to the following optimization problem:
\begin{equation}\label{eqn:special}
\begin{array}{cll}
\defmin_{\alpha\in\calA_{mn},\beta\in\R^{m\times n}\ge0,\gamma\in\R^{m\times n}\ge0}& tr(\alpha\calC^\top)& + \lambda_1\norm{\alpha\bone-\frac{1}{m}\bone}_{G_1}^2 + \lambda_2\norm{\alpha^\top\bone-\frac{1}{n}\bone}_{G_2}^2\\
&&+ \nu_1\norm{\alpha\bone-\frac{1}{m}\bone}_{G_1\odot G_1}^2 + \nu_2\norm{\alpha^\top\bone-\frac{1}{n}\bone}_{G_2\odot G_2}^2\\\\
\textup{s.t.\ \ \ \ \ \ \ \ \ \ \ \ }&& \alpha=\frac{1}{m}G_1\beta^\top, \alpha=\frac{1}{n}\gamma G_2,
\end{array}
\end{equation}
where $\calA_{mn}=\{x\in\R^{m\times n}\ |\ x\ge\bzero,\bone^\top x\bone=1\}$ and  $\lambda_i>0,\nu_i>0$ are the regularization hyper-parameter corresponding to $\Delta_i^2,\zeta_i^2$ in (\ref{eqn:final}).

The above form makes it very attractive for employing the ADMM~\cite{admm} algorithm for three reasons: i) the natural consensus form, ii) the constraints in terms of $\alpha,\beta,\gamma$ variables (apart from the consensus constraints) are simple and motivate efficient mirror descent~\cite{mirror} based solvers for the ADMM updates iii) the consensus in (\ref{eqn:special}) needs to applied only via regularization (and not an exact match), which is naturally facilitated in the ADMM updates. The updates for the ADMM are summarized below:

\begin{equation}\label{eqn:adm1}
\begin{array}{lll}
\alpha^{(k+1)} &\coloneqq &\argmin_{\alpha\in\calA_{mn}} \rho\norm{\alpha + \frac{1}{2}\left(D_1^{(k)}+D_2^{(k)} + \frac{\calC}{\rho} - \frac{\gamma^{(k)} G_2}{n} - \frac{G_1{\beta^{(k)}}^\top}{m}\right)}^2\\
& & \quad\quad\quad\quad + \lambda_1\norm{\alpha\bone-\frac{\bone}{m}}_{G_1}^2 + \lambda_2\norm{\alpha^\top\bone-\frac{\bone}{n}}_{G_2}^2,\nonumber\\
& & \quad\quad\quad\quad + \nu_1\norm{\alpha\bone-\frac{1}{m}\bone}_{G_1\odot G_1}^2 + \nu_2\norm{\alpha^\top\bone-\frac{1}{n}\bone}_{G_2\odot G_2}^2\\
\beta^{(k+1)} &\coloneqq &\argmin_{\beta\geq 0} \norm{\alpha^{(k+1)} + D_1^{(k)} - \frac{G_1\beta^\top}{m}}^2,\nonumber\\
\gamma^{(k+1)} &\coloneqq &\argmin_{\gamma\geq 0} \norm{\alpha^{(k+1)} + D_2^{(k)} - \frac{\gamma G_2}{n}}^2,\nonumber\\
D_1^{(k+1)} & \coloneqq & D_1^{(k)}+ \left(\alpha^{(k+1)} - \frac{G_1{\beta^{(k+1)}}^\top}{m}\right),\nonumber\\ 
D_2^{(k+1)} &\coloneqq & D_2^{(k)} + \left(\alpha^{(k+1)} - \frac{\gamma^{(k+1)} G_2}{n}\right),\nonumber
\end{array}
\end{equation}
%
%
where $D_1$ and $D_2$ are the dual variables corresponding to the constraints $\alpha=(1/m)G_1\beta^\top$ and $\alpha=(1/n)\gamma G_2$ in (\ref{eqn:special}), respectively. The optimization problems with respect to $\alpha,\beta,$ and $\gamma$ can be solved efficiently using popular algorithms like conditional gradient descent, mirror descent, co-ordinate descent, conjugate gradients, etc. Since the convergence rate of these algorithms is either independent or almost independent (logarithmically dependent) on the dimensionality of the problem, the computational cost (after neglecting log factors, if any) of solving for: $\alpha$ is $O(mn)$, $\beta$ is $O(m^2 n)$, and $\gamma$ is $O(mn^2)$. The updates for $D_1$ and $D_2$ have computational costs: $O(m^2 n)$ and $O(mn^2)$. Without loss of generality, if we assume $m\geq n$, the per iteration cost of ADMM is $O(m^3)$.

As noted earlier, in typical cases where the hyper-parameters $\Delta_i$ are small enough, explicit constraints $\beta\ge0,\gamma\ge0$ are not needed, leading to a further simplified formulation:
\begin{equation}\label{eqn:tok}
\begin{array}{cl}
\defmin_{\alpha\in\calA_{mn}}& tr(\alpha\calC^\top) + \lambda_1\norm{\alpha\bone-\frac{1}{m}\bone}_{G_1}^2 + \lambda_2\norm{\alpha^\top\bone-\frac{1}{n}\bone}_{G_2}^2+ \nu_1\norm{\alpha\bone-\frac{1}{m}\bone}_{G_1\odot G_1}^2 + \nu_2\norm{\alpha^\top\bone-\frac{1}{n}\bone}_{G_2\odot G_2}^2\\
\end{array}
\end{equation}
The above is a strongly convex problem with simplex constraints. The frank-wolfe and co-ordinate descent methods are known to lead to extremely scalable algorithms for such cases. For few Frank-Wolfe iterations, the computational cost is $O(mn)$ and the gram matrix computations require $O(m^2),O(n^2)$ respectively. In case $m=n$, this matches the linear complexity of the popular Sinkhorn algorithm used to solve the discrete OT problem.

\section{Related Work}
A popular strategy for performing continuous statistical OT is to simply employ the sample based plug-in estimates for the marginals. This reduces the statistical OT problem to the classical discrete OT problem, for which efficient algorithms exist~\cite{sinkhorn13,Jason19nystrom}. However, the sample complexity of the discrete OT based estimation is plagued with the curse of dimensionality~\cite{nilesweed2019estimation}.

Many approaches for alleviating the curse of dimensionality exist. For e.g.,~\cite{Genevay19sampComp} propose entropic regularization. However, their results (e.g., theorem~3~in~\cite{Genevay19sampComp}) show that the curse of dimensionality is not completely removed, especially if accurate solutions are desired. Empirical results in~\cite[refer Figures  4 and 5]{Feydy2018InterpolatingBO} confirm that the quality of the solution degrades very quickly with entropic regularization.~\cite{nilesweed2019estimation,Forrow19facCoup} make low-rank assumptions, which may not be satisfied in all applications. Further, as per theorem~1~in~\cite{nilesweed2019estimation}, the dependence on $d$ still exists. Similar comments hold for~\cite{hutter2020minimax}, which makes a smoothness assumption.

While the approach of~\cite{Genevay16sgd} efficiently estimates the optimal dual objective, recovering the optimal transport plan from the dual's solution again requires the knowledge of the exact marginals (refer proposition~2.1~in~\cite{Genevay16sgd}). Since estimating distributions in high-dimensional settings is known to be challenging, this alternative is not attractive for applications where the transport plan is required, e.g., domain adaptation~\cite{Courty17domAda},ecological inference~\cite{ecoinf17}, data alignment~\cite{alvarez18wordEmb} etc.


On passing we note that though there are existing works that employ kernels in context of OT~\cite{Genevay16sgd,Perrot16outofsamp,zhangotkernel19,oh2019kernel}, none of them use the notion of kernel embedding of distributions and limit the use of kernels to either function approximation or computing the MMD distance. Though relations betweeen Wasserstein and MMD distance~\cite{Feydy2018InterpolatingBO} exist, none of them explore regularization with MMD distances.


\begin{figure}
\begin{center}
	\begin{minipage}[t]{.32\textwidth}
	\begin{center}
		\includegraphics[width=\textwidth]{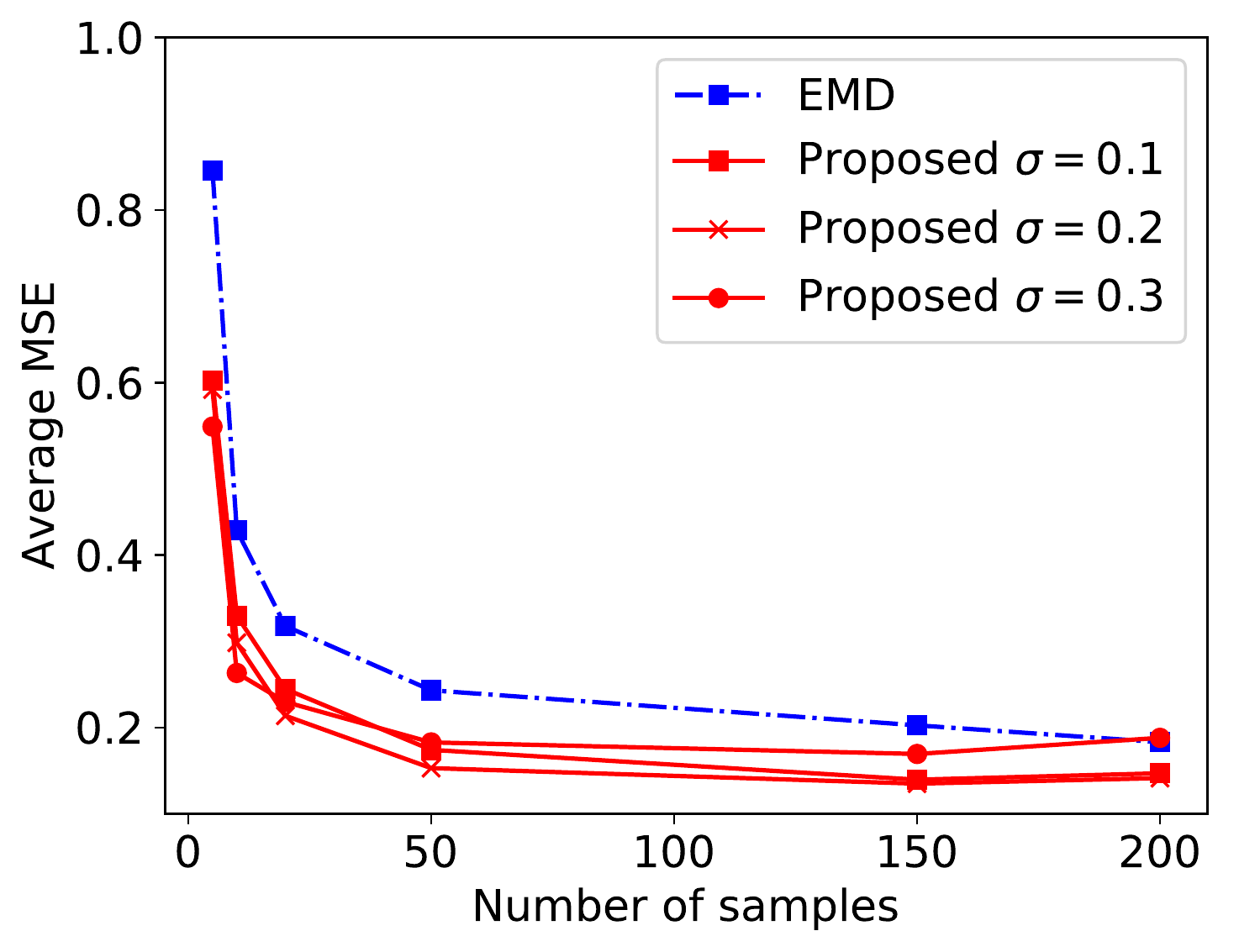}\\
		{\scriptsize  (a) Dimension $d=10$.}
	\end{center} 
	\end{minipage}
	\begin{minipage}[t]{.32\textwidth}
	\begin{center}
		\includegraphics[width=\textwidth]{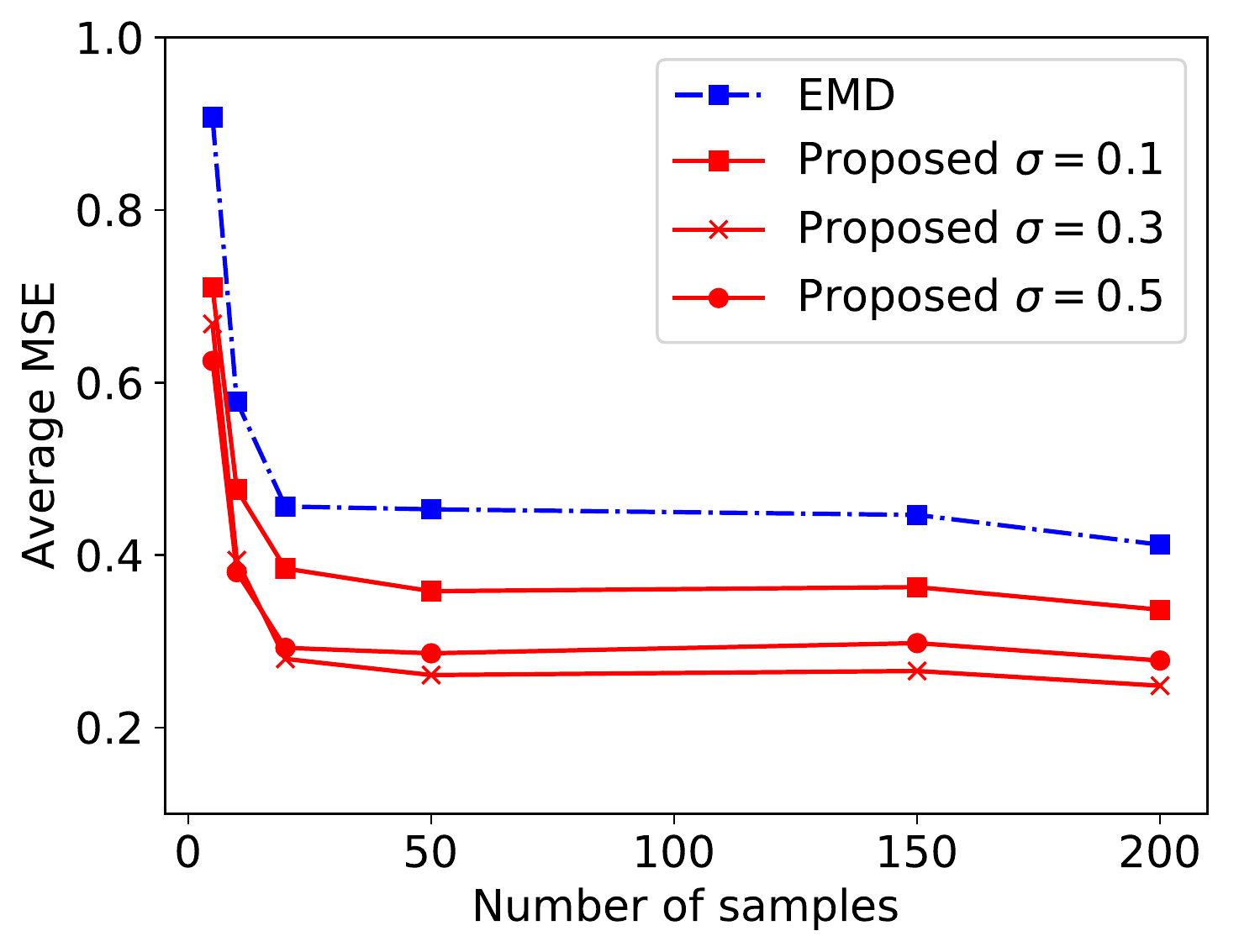}\\
		{\scriptsize  (b) Dimension $d=100$.}
	\end{center} 
	\end{minipage}
	\begin{minipage}[t]{.32\textwidth}
	\begin{center}
		\includegraphics[width=\textwidth]{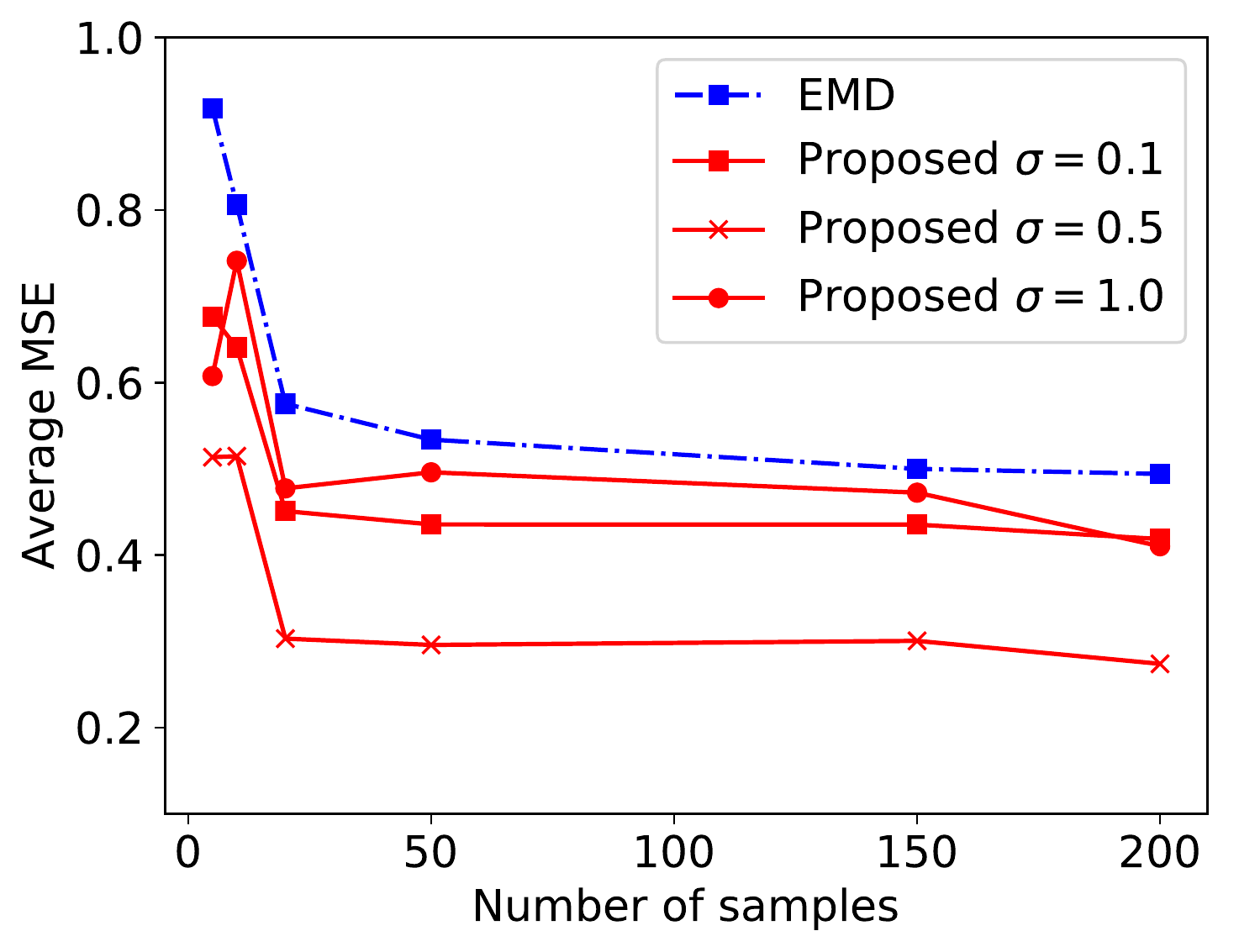}\\
		{\scriptsize  (c) Dimension $d=1000$.}
	\end{center} 
	\end{minipage}
	\caption{Performance on the proposed estimator for the transport map (\ref{eqn:inf}) and the discrete OT estimator, EMD, on the problem of learning the optimal transport map between two multivariate Gaussian distributions. We observe that with suitable $\sigma$, the proposed estimator outperforms EMD especially in higher dimensions. }
\label{fig:mse-results}
\end{center}
\end{figure}

\begin{figure}
\begin{center}
	\begin{minipage}[t]{.32\textwidth}
	\begin{center}
		\includegraphics[width=\textwidth]{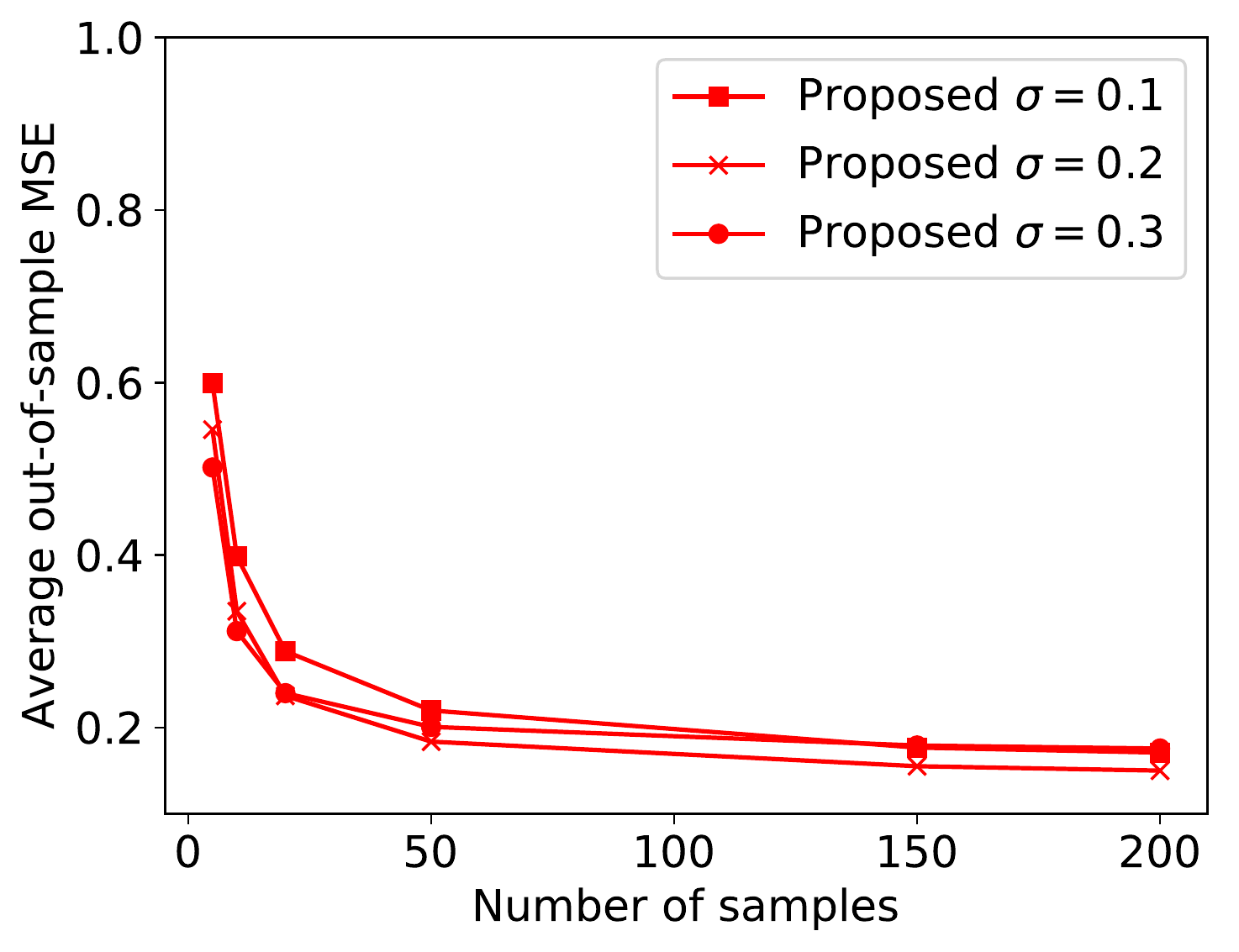}\\
		{\scriptsize  (a) Dimension $d=10$.}
	\end{center} 
	\end{minipage}
	\begin{minipage}[t]{.32\textwidth}
	\begin{center}
		\includegraphics[width=\textwidth]{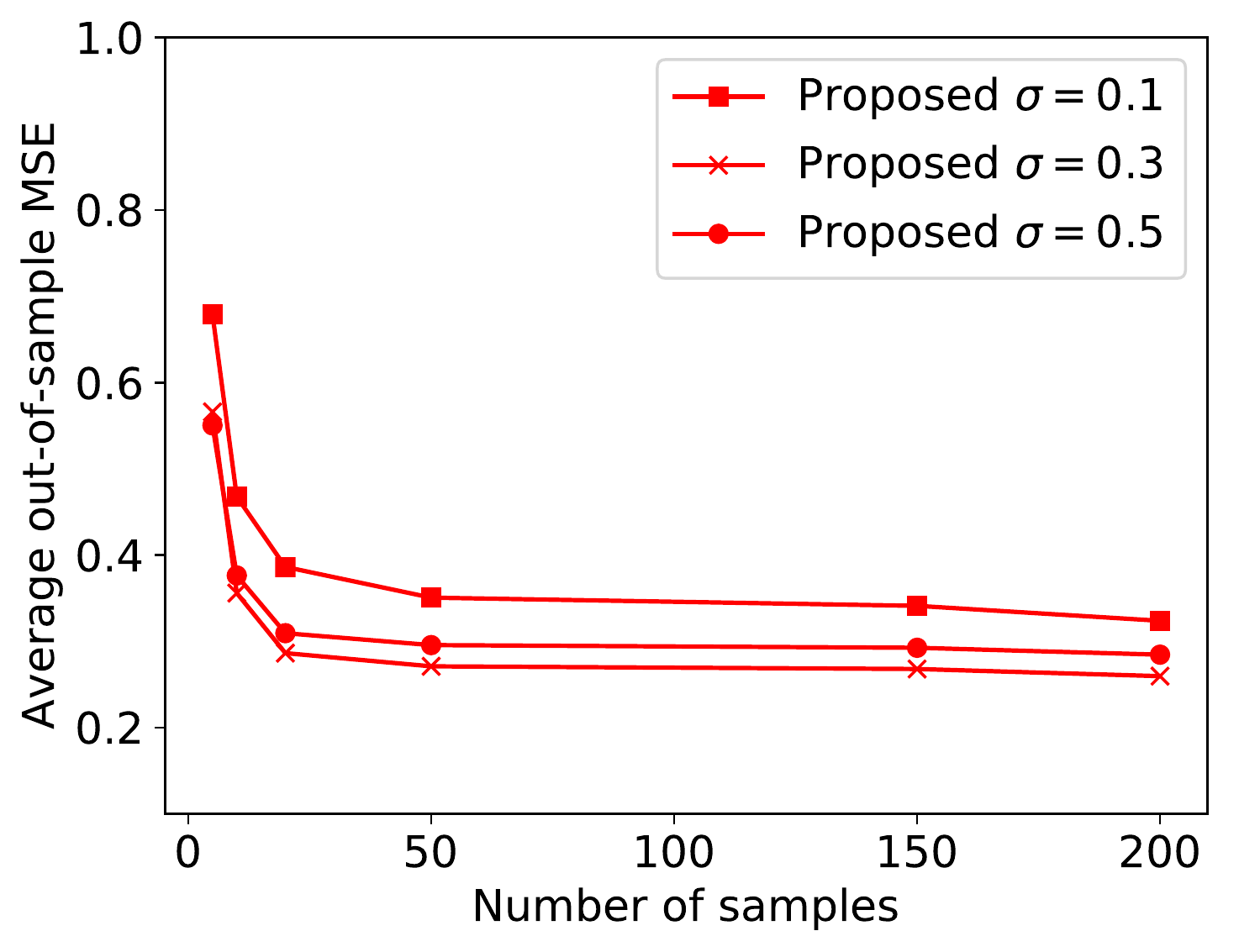}\\
		{\scriptsize  (b) Dimension $d=100$.}
	\end{center} 
	\end{minipage}
	\begin{minipage}[t]{.32\textwidth}
	\begin{center}
		\includegraphics[width=\textwidth]{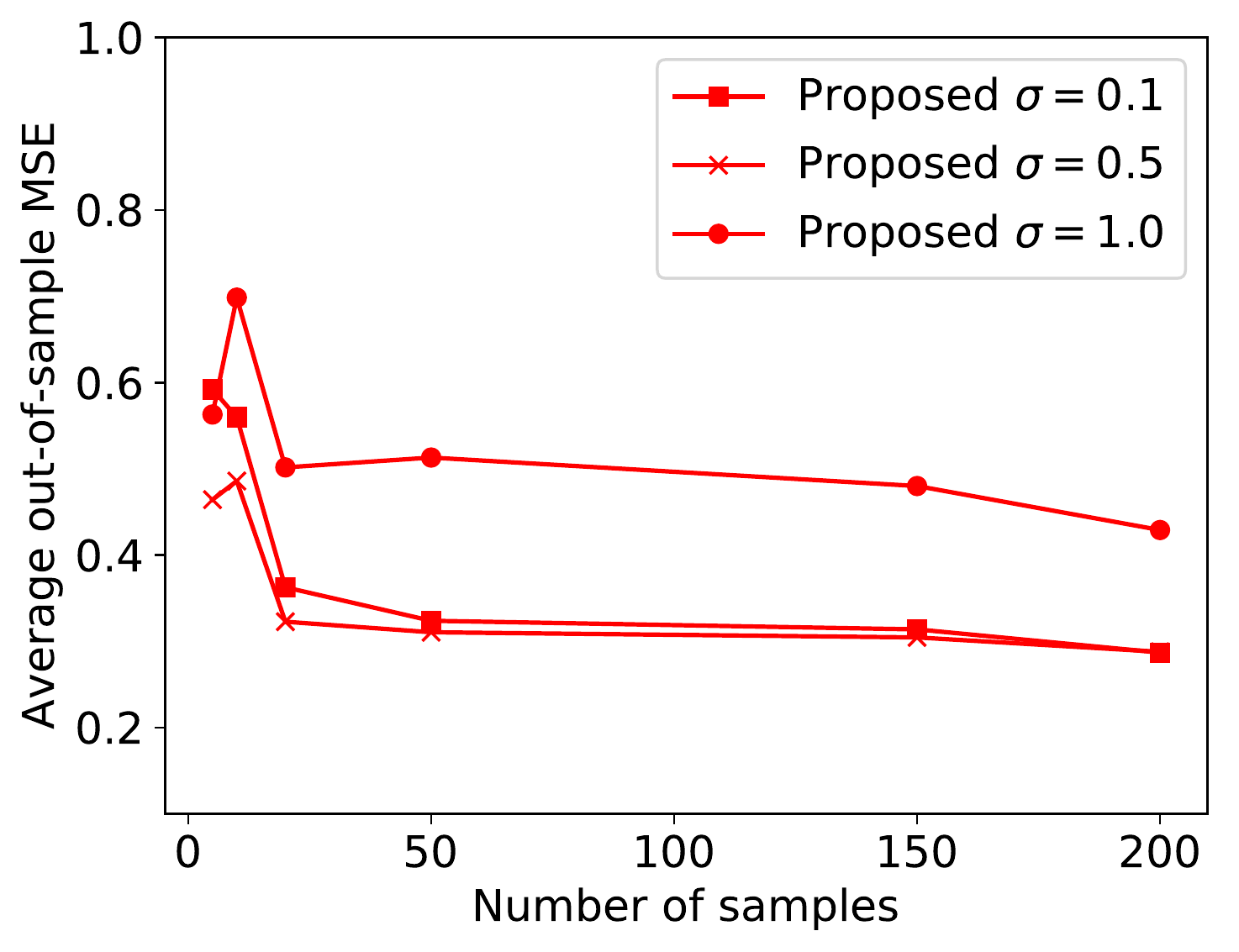}\\
		{\scriptsize  (c) Dimension $d=1000$.}
	\end{center} 
	\end{minipage}
	\caption{Average out-of-sample mean square error (MSE) obtained by the proposed approach on the problem of learning the optimal transport map between two multivariate Gaussian distributions. In general, the average out-of-sample MSE decrease with increasing number of data points sampled to learn the estimator (the x-axis). We also observe that the best results obtained by proposed solution is robust to the dimensionality of the data points. }
\label{fig:mse-results-oos}
\end{center}
\end{figure}


\section{Experiments}\label{sec:experiments}
We evaluate our estimator for the transport map (\ref{eqn:inf}) on both synthetic and real-world datasets. Our code is available at \url{https://www.iith.ac.in/~saketha/research.html}. 

\subsection{Learning OT map between multivariate Gaussian distributions}
The optimal transport map between two Gaussian distributions  $g_{source}=N(m_1,\Sigma_1)$ and $g_{target}=N(m_2,\Sigma_2)$ with squared Euclidean cost has a closed form expression \cite{CompOT} given by $T:x\mapsto m_2 + A(x-m_1)$, 
where $A=\Sigma_1^{-\frac{1}{2}}(\Sigma_1^{\frac{1}{2}}\Sigma_2\Sigma_1^{\frac{1}{2}})^{\frac{1}{2}}\Sigma_1^{-\frac{1}{2}}$. 
We compare the proposed estimator (\ref{eqn:inf}) in terms of the deviation from the optimal transport map. 

\textbf{Experimental setup}: We consider mean zero Gaussian distributions with unit-trace covariances. The covariance matrices are computed as $\Sigma_1=V_1 V_1^\top/\|V_1\|_F$ and  $\Sigma_2=V_2 V_2^\top/\|V_2\|_F$, where $V_1\in\R^{d\times d}$ and $V_2\in\R^{d\times d}$ are generated randomly from the uniform distribution. We experiment with varying dimensions and number of data-points: $d\in\{10,100,1000\}$,  $m\in\{10,20,50,100,150,200\}$, and we set $n=m$ for simplicity. 
For each dimension $d$, we randomly generate a source-target distribution pair. Subsequently, the source and target datasets (of size $m$) are randomly generated from their respective distributions. For a every $(d,m)$, we repeat the experiments five times and report the average mean square error (MSE) results results in Figures \ref{fig:mse-results} and \ref{fig:mse-results-oos}. 
 
\textbf{Methods}: The proposed approach employs the Gaussian kernels, $k(x,z)=\exp(-\|x-z\|^2/2\sigma^2)$. We chose the same $\sigma$ values for the kernels over the source and the target domains ($k_1$ and $k_2$, respectively). Initial experiments indicate that suitable values of $\sigma$ include those that does not yield high condition number of the Gram matrices (i.e, the Gram matrices are not ill-conditioned). In our setup, in general, the condition number of the Gram matrices increase with $\sigma$ for a fixed $d$ and decrease with $d$ for a fixed $\sigma$. The $\sigma$ values used in various experiments are mentioned with the results. As a baseline, we also report the results obtained from the discrete OT estimator, henceforth referred to as EMD, learned via the discrete OT problem (\ref{eqn:discOT}). 


\textbf{Evaluation}: 
For a given data point $x_s$ from the source distribution, a transport map estimator maps $x_s$ to a data point $x_t$ in the target distribution. Such a  mapping obtained from the optimal transport map  (15) is considered as the ground truth. The proposed estimator (8) and the EMD are evaluated in terms of the mean squared error (MSE) from the ground truth.

\textbf{Results}: The results of our first set of experiments are reported in Figures~\ref{fig:mse-results}(a)-(c). We observe that the proposed estimator obtains lower average MSE (and hence better estimation of the transport map) than EMD across different number of samples $m$ and dimensions $d$. The advantage of the proposed estimator over the baseline is more pronounced at higher dimension. 


\textbf{Out-of-sample evaluation}: We also evaluate our estimator's ability to map  out-of-sample data by sampling additional $m_{oos}=200$ points from the source distributions in the above experiments. These source points are not used to learn the estimator and are only used for evaluation during the inference stage. 
The results on out-of-sample dataset, corresponding to the first set of experiments (Figure~\ref{fig:mse-results}) are reported in Figure~\ref{fig:mse-results-oos}. We generate out-of-sample data points for each $(d,s)$ pair, where $d$ is the dimension of the data points and $s$ is the random seed (corresponding to five repetition  discussed earlier). Hence, for a given $(d,s)$ pair, different estimators learned with varying $m$ are evaluated on the same set of out-of-sample data points.

We observe that the performance on out-of-sample data points are similar to the in-sample data points. The average out-of-sample MSE generally decreases with increasing number of (training) samples since a better estimator is learned with more number of samples. Overall, the results illustrate the utility of the proposed approach for out-of-sample estimation. It should be noted that the baseline EMD cannot map out-of-sample data points.

\begin{table*}
\caption{Accuracy obtained on the target domains of the Office-Caltech dataset. The knowledge transfer to the target domain happens via in-sample source data-points, i.e., those source data-points using which the transport plan was learned. }\label{table:officeCaltech}
\centering
\begin{tabular}{lllll}
\toprule
\lhead{Task} & \thead{EMD}  & \thead{OTLin \cite{perrot16a}} & \thead{OTKer \cite{perrot16a}}  & \thead{Proposed}\\
\cmidrule(r){1-1}
\cmidrule(r){2-5}
$A\rightarrow C$ & $80.68\pm1.82$ & $82.92\pm1.41$ & $83.07\pm0.63$ & $	\mathbf{85.24\pm1.95}$\\
$A\rightarrow D$ & $72.66\pm6.58$ & $82.28\pm5.66$ & $82.53\pm3.70$ & $	\mathbf{84.05\pm5.03}$\\
$A\rightarrow W$ & $69.05\pm5.08$ & $\mathbf{77.70\pm3.60}$ & $76.35\pm4.16$ & $ 77.57\pm3.68$\\
$C\rightarrow A$ & $82.61\pm3.45$ & $88.31\pm0.94$ & $88.09\pm1.50$ & $	\mathbf{90.49\pm1.24}$\\
$C\rightarrow D$ & $68.35\pm10.06$ & $\mathbf{79.75\pm6.04}$ & $78.99\pm7.95$ & $ 77.97\pm7.19$\\
$C\rightarrow W$ & $65.54\pm2.74$ & $71.89\pm3.43$ & $70.00\pm3.93$ & $	\mathbf{74.46\pm1.21}$\\
$D\rightarrow A$ & $81.50\pm1.99$ & $88.57\pm1.86$ & $85.23\pm1.71$ & $	\mathbf{91.09\pm1.58}$\\
$D\rightarrow C$ & $76.51\pm2.87$ & $82.17\pm1.70$ & $78.22\pm1.80$ & $	\mathbf{86.52\pm1.19}$\\
$D\rightarrow W$ & $91.89\pm2.96$ & $\mathbf{97.57\pm1.09}$ & $96.35\pm1.10$ & $ 96.35\pm1.95$\\
$W\rightarrow A$ & $71.22\pm1.54$ & $80.00\pm1.74$ & $76.23\pm2.50$ & $	\mathbf{87.54\pm4.02}$\\
$W\rightarrow C$ & $69.55\pm3.18$ & $77.58\pm2.34$ & $73.72\pm2.59$ & $	\mathbf{80.96\pm3.07}$\\
$W\rightarrow D$ & $80.76\pm4.90$ & $\mathbf{97.72\pm1.47}$ & $96.20\pm2.89$ & $ 96.20\pm3.10$\\
\midrule
Average & $75.86\pm1.43$ & $83.87\pm0.38$ & $	82.08\pm0.95$ & $	\mathbf{85.70\pm0.89}$\\
\bottomrule
\end{tabular}
\end{table*}

\subsection{Domain adaptation}
We experiment on the Caltech-Office dataset \cite{gong12a}, which contains images from four domains: Amazon (online retail), the Caltech image dataset, DSLR (images taken from a high resolution DSLR camera), and Webcam (images taken from a webcam). The domains vary with respect to factors such as background, lightning conditions, noise, etc. The number of examples in each domain is: $958$ (Amazon), $1123$ (Caltech), $157$ (DSLR), and $295$ (Webcam). Each domain has images from ten classes and in turn is considered as the source or the target domain. 
We perform multi-class classification in the target domain given the labeled data only from the source domain. Using OT, we first transport the labeled source domain data-points to the target domain and then learn a classifier for the target domain using the adapted source data-points. 
Overall, there are twelve adaptation tasks (e.g., task $A\rightarrow C$ has Amazon as the source and Caltech as the target domain). We employ DeCAF6 features to represent the images \cite{donahue14a,perrot16a,courty17b}.

\textbf{Experimental setup}: For learning transport plan, we randomly select ten images per class for the source domain (eight per class when DSLR is the source, due to its sample size). The remaining samples of the source domain is marked as out-of-sample source data-points. 
The target domain is partitioned equally into training and test sets. The transport map is learned using the source-target training sets. The `in-sample' accuracy is then evaluated on the target's test set. 
We also evaluate the quality of our out-of-sample estimation as follows. Instead of projecting the source training set samples onto the target domain, we project only the out-of-sample (OOS) source data-points and compute the accuracy over the target's test set. It should be noted that the transport model has not been learned on the OOS data-points, such mappings may not be as accurate as the in-sample mapping. The OOS evaluation assesses the downstream  effectiveness of OOS estimation on domain adaptation. Out-of-sample estimation is especially attractive in big data and online applications.
The classification in the target domain is performed using a 1-Nearest Neighbor classifier \cite{gong12a,perrot16a,courty17b}. The above experimentation is performed five times. The average in-sample and out-of-sample accuracy are reported in Tables~\ref{table:officeCaltech}~\&~\ref{table:officeCaltech-oos}, respectively. 

\textbf{Methods}: We compare our approach with EMD, OTLin~\cite{perrot16a}, and OTKer~\cite{perrot16a}. Both OTLin and OTKer aim to solve the discrete optimal transport problem and also learn a transformation approximating the corresponding transport map in a joint optimization framework. OTLin learns a linear transformation while OTKer learns a non-linear transformation (e.g., via Gaussian kernel). The learned transformation allows OTLin and OTKer to perform out-of-sample estimation as well. 
Both OTLin and OTKer employ two regularization parameters. As suggested by their authors~\cite{perrot16a}, both the regularization parameters were chosen from the set $\{10^{-3}, 10^{-2}, 10^{-1}, 10^{0}\}$. It should be noted that best regularization parameters were selected for each task. OTKer additionally requires Gaussian kernel's hyper-parameter $\sigma$, which was chosen from the set $\{0.1, 0.5, 1, 5, 10\}$. We use the Python Optimal Transport (POT) library (\url{https://github.com/PythonOT/POT}) implementations of OTLin and OTKer in our experiments. For the proposed approach, as in the previous experiments, we chose the Gaussian kernels and have same $\sigma$ values for the kernels over the source and the target domains. The $\sigma$ for our approach was also chosen from the set $\{0.1, 0.5, 1\}$. 

\textbf{Results}: We observe from Tables~\ref{table:officeCaltech}~\&~\ref{table:officeCaltech-oos} that the proposed approach outperforms the baselines, obtaining the best in-sample and out-of-sample (OOS) accuracy. 
As discussed, the in-sample accuracy is likely to be better than out-of-sample accuracy (for any approach). 
Interestingly, for a few tasks with Amazon and Caltech as the source domains, the OOS accuracy of our approach is comparable to our in-sample accuracy. In these domains, the OOS set is larger than the training set. The proposed OOS estimation is able to exploit this and provide an effective knowledge transfer. Conversely, we observe a drop in our OOS accuracy (when compared with the corresponding in-sample accuracy) in tasks with DSLR and Webcam as the source domains since the size of OOS set is quite small and hence lesser potential for knowledge transfer. On the other hand, OTLin suffers a significant drop in OOS performance, likely due the the overfitting of the learned linear transformation on the source training points. While OTKer has better OOS performance than OTLin, it has more variance between in-sample and out-of-sample performance than the proposed approach. 

\begin{table*}
\caption{Accuracy obtained on the target domains of the Office-Caltech dataset. The knowledge transfer to the target domain happens via out-of-sample source data-points, i.e., those source data-points which were not used for learning the transport plan. }\label{table:officeCaltech-oos}
\centering
\begin{tabular}{llll}
\toprule
\lhead{Task} & \thead{OTLin \cite{perrot16a}} & \thead{OTKer \cite{perrot16a}}  & \thead{Proposed}\\
\cmidrule(r){1-1}
\cmidrule(r){2-4}
$A\rightarrow C$ & $56.75\pm2.94$ & $79.11\pm2.76$ & $	\mathbf{84.42\pm1.93}$\\
$A\rightarrow D$ & $79.49\pm1.86$ & $82.79\pm2.06$ & $	\mathbf{84.81\pm1.55}$\\
$A\rightarrow W$ & $55.41\pm6.60$ & $76.35\pm1.66$ & $	\mathbf{82.16\pm2.93}$\\
$C\rightarrow A$ & $87.79\pm3.28$ & $84.54\pm2.78$ & $	\mathbf{90.71\pm1.11}$\\
$C\rightarrow D$ & $\mathbf{81.01\pm3.75}$ & $74.94\pm3.53$ & $78.73\pm5.76$\\
$C\rightarrow W$ & $70.00\pm3.64$ & $68.11\pm1.16$ & $	\mathbf{75.27\pm4.93}$\\
$D\rightarrow A$ & $64.53\pm5.01$ & $81.95\pm2.69$ & $	\mathbf{87.56\pm3.10}$\\
$D\rightarrow C$ & $43.67\pm4.80$ & $72.79\pm3.04$ & $	\mathbf{81.33\pm1.56}$\\
$D\rightarrow W$ & $\mathbf{90.04\pm2.81}$ & $82.02\pm0.72$ & $ 89.60\pm2.13$\\
$W\rightarrow A$ & $60.09\pm4.77$ & $73.88\pm2.83$ & $	\mathbf{79.10\pm4.68}$\\
$W\rightarrow C$ & $49.34\pm8.78$ & $63.17\pm4.14$ & $	\mathbf{76.83\pm1.91}$\\
$W\rightarrow D$ & $\mathbf{95.95\pm1.86}$ & $90.89\pm1.68$ & $93.16\pm1.44$\\
\midrule
Average & $69.51\pm2.70$ & $77.54\pm0.66$ & $\mathbf{83.64\pm0.38}$\\
\bottomrule
\end{tabular}
\end{table*}

\section{Conclusions}\label{sec:concsum}
The idea of employing kernel embeddings of distributions in OT seems promising, especially in the continuous case. It not only leads to sample complexities that are dimension-free, but also provides a new regularization scheme based on MMD distances, which is complementary to existing $\phi$-divergence based regularization.

While the optimal solution of the proposed MMD regularized formulation recovers the transport plan, the objective value does not seem to have any special use. On the contrary, it has been shown that with entropic, $\phi$-divergence based regularizations the optimal objectives lead to notions of Sinkhorn divergences~\cite{Feydy2018InterpolatingBO} and Hillinger-Kantorovich metrics~\cite{Liero2018}. We make an initial observation that in the special case in section~\ref{sec:special}, the objective in (\ref{eqn:special}), resembles that defining the Hillinger-Kantorovich metrics very closely. Hence, we conjecture that our optimal objective in this special case may also define a new family of metrics. However, we postpone such connections (if any) to future work.
\appendix
\section{Proof for Theorem \ref{thm:sampcomp}}\label{app:samplecomplexity}
\begin{proof}
Let $\hath$ denote the objective in (\ref{eqn:empemb0}), when written in Tikhonov form, as a function of variables $\calU\in\calE_{21},\calU_1\in\calL_{12},\calU_2\in\calL_{21}$ and let $h$ denote that when the true embeddings are employed instead of their estimates. In particular, we have $\hath\left(\hatcalU,\hatcalU_1,\hatcalU_2\right)= g\left(\hatc,\hatmu_s,\hatmu_t,\hatSigma_{ss},\hatSigma_{tt}\right), h\left(\calU^*,\calU^*_1,\calU^*_2\right)= g\left(c,\mu_s,\mu_t,\Sigma_{ss},\Sigma_{tt}\right)$, where $\hatcalU,\hatcalU_1,\hatcalU_2$ and $\calU^*,\calU^*_1,\calU^*_2$ are optimal solutions to respective problems.

We begin by noting that the feasibility set of (\ref{eqn:empemb0}) is bounded. This is because: i) the set $\calE_{21}$ is bounded. This is true as $\calU\in\calE_{21}\Rightarrow$ there exists $p\in\calM^1(\calX\times\calY)$ such that $\|\calU\|=\|\E_{(X,Y)\sim p}\left[\phi_1(X)\otimes\phi_2(Y)\right]\|\le\E_{(X,Y)\sim p}\left[\|\phi_1(X)\otimes\phi_2(Y)\|\right]=1$. The first inequality follows from Jensens inequality and the second equality is true  for any bounded kernel like Gaussian and the Kroncker Delta. ii) By triangle inequality,  $\left|\|\calU\|-\|\hatSigma_{ss}\calU_1\|\right|\le\left\|\calU-\hatSigma_{ss}\calU_1^\top\right\|\le\vartheta_1$. This shows that the set of all feasible $\hatSigma_{ss}\calU_1$ is bounded, since $\calU$ is itself bounded in the feasibility set. Now, since $maxeig(\hatSigma_{ss})=maxeig(G_2)/n\le tr(G_2)/n=1$ (again true for Kronecker and Gaussian kernels), we obtain that set of all feasible $\calU_1$ is also bounded. Similarly, set of all feasible $\calU_2$ is bounded. Accordingly, we define $\calB\left(\vartheta_1,\vartheta_2\right)\equiv\left\{\left(\calU\in\calE_{21},\calU_1\in\calL_{12},\calU_2\in\calL_{21}\right)\ \left|\ \|\calU\|\le1,\|\calU_1\|\le1+\vartheta_1,\|\calU_2\|\le1+\vartheta_2\right.\right\}$. By the above argument, it is clear that there is no loss of generality in further restricting the search space to that with intersection with this bounded set, $\calB\left(\vartheta_1,\vartheta_2\right)$, and always $\left(\calU^*,\calU_1^*,\calU_2^*\right),\left(\hatcalU,\hatcalU_1,\hatcalU_2\right)\in\calB\left(\vartheta_1,\vartheta_2\right)$ for any $m,n\in\N$.

The rest of the proof follows from the claim below:
\begin{clm}
The uniform bound:
$$\max_{\left(\calU,\calU_1,\calU_2\right)\in\calB\left(\vartheta_1,\vartheta_2\right)}\left|\hath\left(\calU,\calU_1,\calU_2\right)-h\left(\calU,\calU_1,\calU_2\right)\right|\le O\left(1/\sqrt{\min(m,n)}\right)$$
holds, where the constants in the RHS are dimension-free.
\end{clm}
\begin{proof}
 Now consider the Tikhonov regularized form of (\ref{eqn:empemb0}). Then, one of the term in the objective is $\|\calU_1\hatmu_s-\hatmu_t\|\le\|\calU_1\left(\hatmu_s-\mu_s\right)\|+\left\|\mu_t-\hatmu_t\right\|+\left\|\calU_1\mu_s-\mu_t\right\|$, which is less than $\|\calU_1\mu_s-\mu_t\|+O(\frac{1}{\sqrt{p}})$ with high probability. Here, $p=\min(m,n)$. The first inequality is by triangle inequality and the second is the crucial one that follows from sample complexity of kernel mean embeddings (see theorem 2 in~\cite{smola07kmm}), and the boundedness of $\|\calU_1\|$. Also, the constants in $O(\frac{1}{\sqrt{p}})$ are independent of samples, variables and dimensions. By symmetry, we also have with high probability that  $\|\calU_1\mu_s-\mu_t\|\le\|\calU_1\hatmu_s-\hatmu_t\|+O(\frac{1}{\sqrt{p}})$. Hence, with high probability, uniformly over the feasibility set,  $\left|\|\calU_1\mu_s-\mu_t\|-\|\calU_1\hatmu_s-\hatmu_t\|\right|\le O(\frac{1}{\sqrt{p}})$. Analogous arguments hold for the other regularizer terms too. Now, we analyze the linear objective term. By Jensen's inequality, $|\langle\hatc,\calU\rangle-\langle c,\calU\rangle|\le\sqrt{\E_u\left[\left(\hatc-c\right)^2\right]}$, where $u$ is the measure corresponding to $\calU$. Let $\baru$ denote the product measure of the given marginals. It is easy to see that $\left\{\left(x_i,y_j\right)\ |\ i\in1,\ldots,m,j\in1,\ldots,n\right\}$ is an iid set of samples from $\baru$. By eqn. (4) in theorem~3.1~in~\cite{Ali08approx} and lemma~1 in~\cite{alibenj08}, we have $\sqrt{\E_u\left[\left(\hatc-c\right)^2\right]}\le\frac{\|c\|_{\baru}}{\sqrt{mn}}\left(1+\sqrt{2\log(\frac{1}{\delta})}\right)$, with probability $\delta$. Here, $\|\cdot\|_{\baru}$ is same as that defined in section III of~\cite{Ali08approx}, and theorem~3.1 in~\cite{Ali08approx} applies to our case as we assumed normalized kernels. In particular, this bound is independent of dimensions and $\calU$. To summarize, we have, $|\langle\hatc,\calU\rangle-\langle c,\calU\rangle|\le O(\frac{1}{\sqrt{mn}})$. Finally, again by triangle inequality, $\left|\hath\left(\calU,\calU_1,\calU_2\right)-h\left(\calU,\calU_1,\calU_2\right)\right|$ is less than the sum of deviations in each of the terms detailed above. Since each of these deviations is upper bounded uniformly by $O\left(\frac{1}{\sqrt{p}}\right)$, the claim is proved.
\end{proof}
The proof of the theorem then follows from standard arguments: $\hath\left(\hatcalU,\hatcalU_1,\hatcalU_2\right)-h\left(\calU^*,\calU_1^*,\calU_2^*\right)\le h\left(\hatcalU,\hatcalU_1,\hatcalU_2\right)-h\left(\calU^*,\calU_1^*,\calU_2^*\right)+\max_{\left(\calU,\calU_1,\calU_2\right)\in\calB\left(\vartheta_1,\vartheta_2\right)}\hath\left(\calU,\calU_1,\calU_2\right)-h\left(\calU,\calU_1,\calU_2\right)\le h\left(\hatcalU,\hatcalU_1,\hatcalU_2\right)-h\left(\calU^*,\calU_1^*,\calU_2^*\right)+O\left(1/\sqrt{\min(m,n)}\right)$ by the claim. Now, the estimation error, $h\left(\hatcalU,\hatcalU_1,\hatcalU_2\right)-h\left(\calU^*,\calU_1^*,\calU_2^*\right)$, which is non-negative, is equal to $\left(\hath\left(\hatcalU,\hatcalU_1,\hatcalU_2\right)-\hath\left(\calU^*,\calU_1^*,\calU_2^*\right)\right)+\left(h\left(\hatcalU,\hatcalU_1,\hatcalU_2\right)-\hath\left(\hatcalU,\hatcalU_1,\hatcalU_2\right)\right)+\left(\hath\left(\calU^*,\calU_1^*,\calU_2^*\right)-h\left(\calU^*,\calU_1^*,\calU_2^*\right)\right)\le O\left(1/\sqrt{\min(m,n)}\right)$. The last inequality follows from the claim and the definition of $\left(\hatcalU,\hatcalU_1,\hatcalU_2\right)$ that it minimizes $\hath$. Analogous arguments give $h\left(\calU^*,\calU_1^*,\calU_2^*\right)-\hath\left(\hatcalU,\hatcalU_1,\hatcalU_2\right)\le O\left(1/\sqrt{\min(m,n)}\right)$. This not only completes the proof but also shows that the estimation error also decays with rate that is dimension-free.
\end{proof}
\section{Proof of representer theorem}\label{app:repthm}
\begin{proof}
Without loss of generality, we consider the parameterization: $\calU^\alpha=\sum_{i=1}^m\sum_{j=1}^n\alpha_{ij}\phi_1(x_i)\otimes\phi_2(y_j) + \calU^\perp,\ \calU_1^\beta=\sum_{i=1}^m\sum_{j=1}^n\beta_{ji}\phi_2(y_j)\otimes\phi_1(x_i) + \calU_1^\perp, \calU_2^\gamma=\sum_{i=1}^m\sum_{j=1}^n\gamma_{ij}\phi_1(x_i)\otimes\phi_2(y_j) + \calU_2^\perp,\ \Sigma^\xi_1=\sum_{i=1}^m\xi_{i}\phi_1(x_i)\otimes\phi_1(x_i) + \Sigma_1^\perp,\Sigma^\theta_2=\sum_{j=1}^n\theta_{j}\phi_2(y_j)\otimes\phi_2(y_j) + \Sigma_2^\perp,$ where $\calU^\perp,\calU_1^\perp,\calU_2^\perp,\Sigma_1^\perp,\Sigma_2^\perp$ are the respective orthogonal complements. Let us include the constraint that the underlying measure for $\calU^\alpha,\Sigma_1^\xi,\Sigma_2^\theta$ is the same, as a domain/feasibility constraint.

It is easy to see that the objective as well as the first two inequalities in (\ref{eqn:empemb0}) do not involve the orthogonal complements. Also the term $\left\|\calU-\hatSigma_{ss}\calU_1^\top\right\|^2_{\calH_2\otimes\calH_1}$ can be written as sum of a term not involving the orthogonal complements and $\|\calU^\perp-\hatSigma_{ss}\left(\calU_1^\perp\right)^\top\|^2_{\calH_2\otimes\calH_1}$. Like-wise the each of the terms $\left\|\calU-\calU_2\hatSigma_{tt}\right\|^2_{\calH_2\otimes\calH_1}, \|\Sigma_1^\xi-\Sigma_{ss}\|_{\calH_1\otimes\calH_1}^2,\|\Sigma_2^\xi-\Sigma_{tt}\|_{\calH_2\otimes\calH_2}^2$ can be written as a sum of a term not involving the orthogonal complements, and the terms $\|\calU^\perp-\calU_2^\perp\hatSigma_{tt}\|^2_{\calH_2\otimes\calH_1},\|\Sigma_1^\perp\|^2_{\calH_1\otimes\calH_1},\|\Sigma_2^\perp\|^2_{\calH_2\otimes\calH_2}$ respectively.

Now re-writing (\ref{eqn:empemb0}), where all the norm constraints are equivalently replaced by the norm-squared constraints, in Tikhonov regularization form reads as: $\min_{f\in\calS\subseteq\calH}\hatcalR\left[f\right]+\Omega[f]$, where $f=(\calU,\calU_1,\calU_2,\Sigma_1,\Sigma_2),\calH=\left(\calH_2\otimes\calH_1\right)\oplus\left(\calH_1\otimes\calH_2\right)\oplus\left(\calH_2\otimes\calH_1\right)\oplus\left(\calH_1\otimes\calH_1\right)\oplus\left(\calH_2\otimes\calH_2\right)$, and $\calS=\left\{(\calU,\calU_1,\calU_2,\Sigma_1,\Sigma_2)\ \left|\ \begin{array}{c}\calU\in\calE_{21},\calU_1\in\calL_{12},\calU_2\in\calL_{21},\Sigma_1\in\calL_{11},\Sigma_2\in\calL_{22},\\ \textup{ underlying measure for $\calU^\alpha,\Sigma_1^\xi,\Sigma_2^\theta$ is the same}\end{array}\right.\right\}$, $\Omega[f]\equiv\|\calU^\perp-\hatSigma_{ss}\left(\calU_1^\perp\right)^\top\|^2_{\calH_2\otimes\calH_1}+\|\calU^\perp-\calU_2^\perp\hatSigma_{tt}\|^2_{\calH_2\otimes\calH_1}$ and $\hatcalR[f]$ is the remaining objective that does not involve the orthogonal complements. Also, let $\hatcalS\subset\calS$ denote $\left\{f=(\calU,\calU_1,\calU_2,\Sigma_1,\Sigma_2)\in\calS\ |\ \calU^\perp=0,\calU_1^\perp=0,\calU^\perp_2=0\right\}$. Note that if $=(\calU^\alpha,\calU_1^\beta,\calU_2^\gamma,\Sigma_1^\xi,\Sigma_2^\theta)\in\hatcalS$, then $\Sigma_1^\perp=\Sigma_2^\perp=0$ and $\xi_i=\sum_{j=1}^n\alpha_{ij},\theta_j=\sum_{i=1}^m\alpha_{ij}$. This is again because of one-to-one correspondences implied by the characterstic/universal kernels. Let $\Pi_{\hatcalS}$ denote the projection onto $\hatcalS$. Now, for any $f\in\calH$, we have that: $\hatcalR[\Pi_{\hatcalS}(f)]=\hatcalR[f]$ and more importantly, $0=\Omega[\Pi_{\hatcalS}(f)]\le\Omega[f]$. Consider the following argument\footnote{See also~\cite{Bagnell15} for similar a argument.}: $\min_{f\in\calS\subset\calH}\hatcalR[f]+\Omega[f]\le \min_{f\in\hatcalS\subset\calH}\hatcalR[f]+\Omega[f]= \min_{f\in\calS\subset\calH}\hatcalR[\Pi_{\hatcalS}(f)]+\Omega[\Pi_{\hatcalS}(f)]\le\min_{f\in\calS\subset\calH}\hatcalR[f]+\Omega[f].$ This proves that the orthogonal complements are all zero at optimality.

Now, let $\calL\equiv\left\{\calU^\alpha\ |\ \alpha\in\R^{m\times n}\right\}$, $\calP\equiv\left\{\calU^\alpha\ |\ \alpha\in\R^{m\times n},\alpha\ge0,\bone^\top\alpha=1\right\}$ and $\calA\equiv\left\{\sum_{i=1}^{m^\prime}\sum_{j=1}^{n^\prime}\alpha_{ij}\phi_1(x_i^\prime)\otimes\phi_2(y_j^\prime)\ |\ \alpha\in\R^{m^\prime\times n^\prime},\alpha\ge0,\bone^\top\alpha=1, x_i^\prime\in\calX,y_j^\prime\in\calY,m^\prime,n^\prime\in\N\right\}$. Then, the only thing left to be shown is that $\calE_{21}\cap\calL=\calP$. While $\calP\subseteq\calE_{21}\cap\calL$ is trivial. The converse is true because of the following facts: 
\begin{enumerate}
    \item $\calE_{21}=cl(\calA)$, where $cl$ denotes the set closure. While  $cl(\calA)\subseteq\calE_{21}$ is trivial, the converse follows from the convergence of average of sample embeddings to the true embedding~(see theorem 2 in~\cite{smola07kmm}).
    \item if $\calU\in\calL\backslash\calP$, then $\calU\notin\calA$. This is because the expansion of embeddings in RKHS of a universal kernel are unique. Also, $\calU\in\calL, \calU\notin\calA\Rightarrow\calU\notin cl(\calA)$.
\end{enumerate}
\end{proof}
\section{Proof of Theorem \ref{thm:sgd}}\label{app:sgd}
\begin{proof}
Firstly, by theorem~\ref{thm:sampcomp}, for low enough hyper-parameters we know that the conditional operator obtained by solving (\ref{eqn:final}) are consistent with dimension-free sample complexity. Hence the Barycentric-projection problem is nothing but a stochastic optimization problem with samples as $y_j$ with likelihood $\sum_{j=1}^n\left(\beta^*_{ji}k_1\left(x_i,x\right)\right)$. Using sampling with replacement, these can be converted to $m^\prime$ iid samples with uniform likelihood. Since the cost is assumed to be a metric or it's power greater than unity, the stochastic optimization problem is infact convex wrt $y$. Since the domain $\calY$ is bounded, it is also Lipschitz continuous wrt. $y$. Hence by (7), theorem 3 in~\cite{ShalevShwartz2009StochasticCO}, the estimation error in optimal transport map when solved by SGD is $O(1/\sqrt{m^\prime})$ and remains dimension-free.
\end{proof}
\section*{Acknowledgements}
JSN would like to thank Manohar Kaul and Jatin Chauhan for initial discussions. Part of this work was done while JSN was visiting Microsoft IDC, Hyderabad, and JSN thanks Microsoft for the visiting opportunity. PJ would like to thank Bamdev Mishra for useful discussions.



\bibliography{main}
\bibliographystyle{plain}

\end{document}